\documentclass[letterpaper, 10 pt, conference]{ieeeconf}  

\IEEEoverridecommandlockouts                              

\usepackage{graphics} 
\usepackage{epsfig} 
\usepackage{times} 
\usepackage{amsmath} 
\usepackage{amssymb}  
\usepackage[linesnumbered,boxed,ruled,commentsnumbered,noend]{algorithm2e}
\usepackage{multirow}
\usepackage{booktabs}
\usepackage{threeparttable}

\newtheorem{lemma}{Lemma}

\newtheorem{definition}{Definition}

\DeclareMathOperator*{\argmin}{arg\,min}

\title{\LARGE \bf
Homotopy-aware Multi-agent Navigation via Distributed Model Predictive Control
}

\author{ Haoze Dong$^{1}$, Meng Guo$^{1}$, Chengyi He$^{2}$ and  Zhongkui Li$^{1}$
\thanks{
The authors are with $^{1}$the School of Advanced Manufacturing and Robotics, Peking University, Beijing 100871, China; and $^{2}$the School of Computer Science and Engineering, Beihang University, Beijing 100191, China.
Corresponding author: Zhongkui Li, {\tt\small zhongkli@pku.edu.cn}. 
}
}

\begin{document}

\maketitle
\thispagestyle{empty}
\pagestyle{empty}

\begin{abstract}

Multi-agent trajectory planning requires ensuring both safety and efficiency, yet deadlocks remain a significant challenge, especially in obstacle-dense environments.
Such deadlocks frequently occur when multiple agents attempt to traverse the same long and narrow corridor simultaneously.
To address this, we propose a novel distributed trajectory planning framework that bridges the gap between global path and local trajectory cooperation.
At the global level, a homotopy-aware optimal path planning algorithm is proposed, which fully leverages the topological structure of the environment. 
A reference path is chosen from distinct homotopy classes by considering both its spatial and temporal properties, leading to improved coordination among agents globally.
At the local level, a model predictive control-based trajectory optimization method is used to generate dynamically feasible and collision-free trajectories.
Additionally, an online replanning strategy ensures its adaptability to dynamic environments. 
Simulations and experiments validate the effectiveness of our approach in mitigating deadlocks.
Ablation studies demonstrate that by incorporating time-aware homotopic properties into the underlying global paths, our method can significantly reduce deadlocks and improve the average success rate from 4\%-13\% to over 90\% in randomly generated dense scenarios.

\end{abstract}

\section*{Supplementary Materials}

\noindent \textbf{Code \& Video:} https://github.com/HauserDong/HomoMPC


\section{Introduction}

Multi-agent trajectory planning (MATP) is a fundamental problem in which agents must avoid collisions with both other agents and obstacles while reaching their targets.
In MATP, safety and efficiency are two critical factors.
While significant progress has been made in ensuring safety guarantees, as discussed in the literature \cite{tordesillas2021mader, park2022online}, the efficiency of MATP remains challenged by issues such as local congestion and the deadlock problem where agents block each other indefinitely, preventing further progress~\cite{grover2023before, chen2023multi}.

\subsection{Related Work}

\begin{figure}[t]
        \centering
        \includegraphics[width=\linewidth]{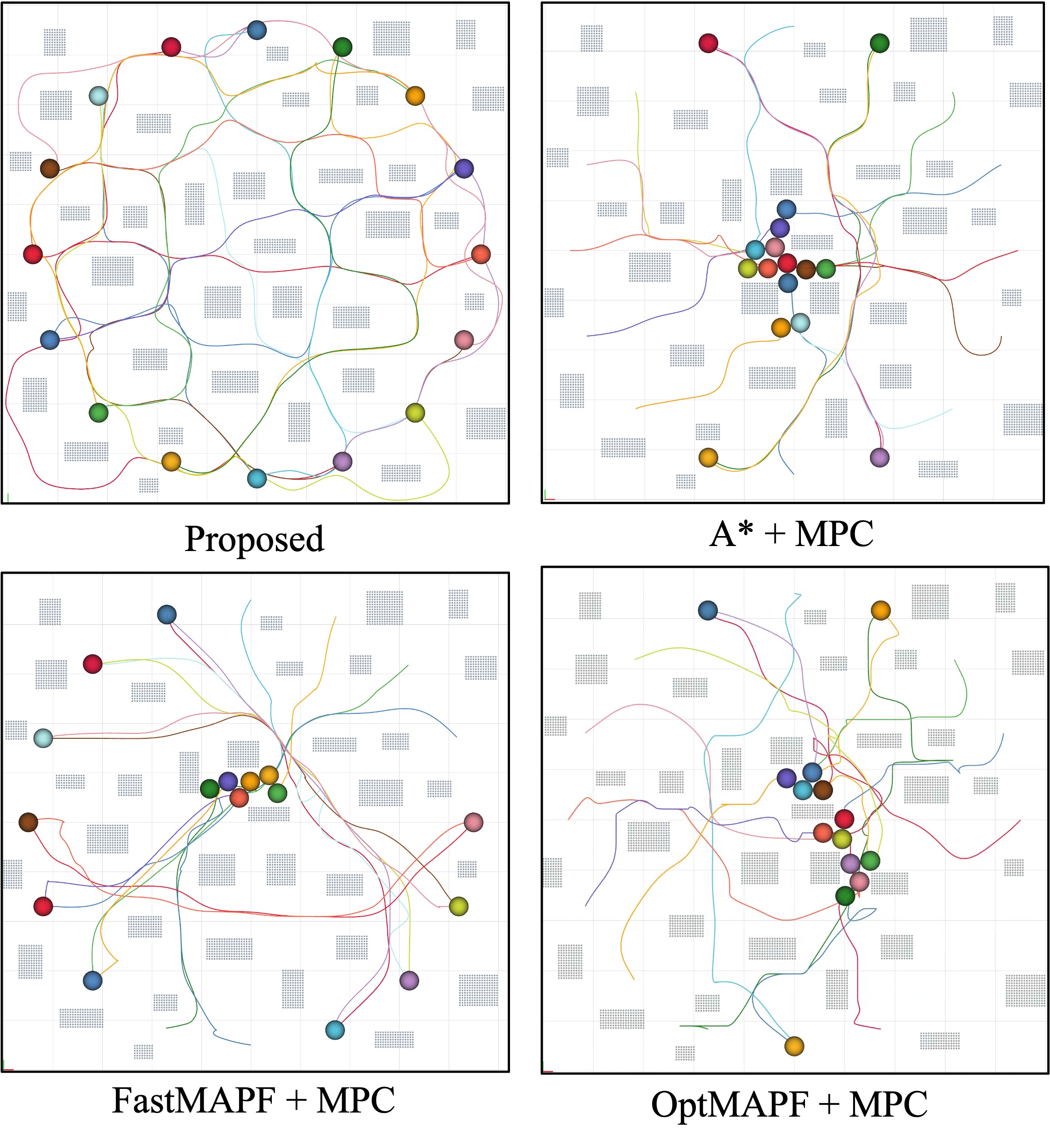}
        \caption{Trajectory comparison of different global path planning methods combined with the same local trajectory planner. Our method leverages the topological structure to achieve spatio-temporal coordination.}
        \label{fig:compare-trajectory}
        \vspace{-0.4cm}
\end{figure}

For MATP problem, a classic \emph{Global} and \emph{Local} two-layer framework is widely adopted \cite{zhou2021raptor, park2023decentralized, cavorsi2023multirobot}.
In the global layer, a reference path is first planned to guide the agent through obstacles.
Then, in the local layer, agents generate trajectories based on the path and safety constraints.

The most straightforward methods directly adopt a single-agent path planning (SAPP) algorithm to generate the global path \cite{tordesillas2021mader, chen2023multi}.
However, the SAPP algorithm finds the optimal path for each individual agent, which may not be the best choice for the whole system.
Local congestion may easily occur, especially in obstacle-dense environments, leading to inefficient path planning.
To improve global cooperation, multi-agent path finding (MAPF) algorithms offer a promising approach \cite{stern2019multi}.
Nevertheless, they still face limitations.
First, MAPF algorithms require strict time-stamped execution, which is impractical in real-world scenarios where agents have dynamic constraints and cope with uncertainties.
Second, although several works have been proposed to integrate MAPF with trajectory optimization \cite{park2023decentralized, hou2022enhanced}, the topological structure of the environment is not considered in MAPF, possibly leading multiple agents to converge in the same region and thus exacerbate local congestion.

Homotopy-based studies have already incorporated the environment's topological structure into trajectory planning.
Two trajectories that connect the same start and goal belong to the same homotopy class if one can be continuously deformed into the other without colliding with any obstacle \cite{bhattacharya2012topological}. 
Otherwise, they belong to different homotopy classes.
For a single agent, distinct homotopy classes of trajectories can be generated to avoid local minima in unknown \cite{zhou2021raptor} or dynamic environments \cite{de2024topology}.
As for multi-agent system, homotopic path set planning is employed to plan the paths for an inseparable group of agents 
\cite{huang2024homotopic,mao2024optimal}.
Under MATP, \cite{zhou2021ego} generates distinct homotopy trajectories by adjusting the initial colliding trajectory around obstacles, then selects the one with the lowest cost for swarm collision avoidance. However, this local homotopy cooperation may not be enough to prevent global congestion.
\cite{kasaura2023homotopy} proposes a centralized framework for generating multiple homotopically distinct solutions to the MAPF problem, followed by their optimization. This approach entails substantial offline computational overhead.

\subsection{Our Method}

This paper presents a novel distributed
trajectory planning framework that enables agents to cooperatively navigate in obstacle-dense environments.
At the global level, we propose a homotopy-aware optimal path planning algorithm that fully exploits the topological structure of the environment and considers potential spatio-temporal conflicts with other agents.
To realize this, a complete passage detection mechanism is introduced to identify passages with their geometric details, while passage time maps are calculated to characterize the temporal behavior of agents passing through these passages.
The proposed algorithm implicitly selects the path that minimizes potential spatio-temporal conflicts.
At local level, each agent employs a model predictive control (MPC)-based trajectory optimization method to generate dynamically feasible and collision-free trajectories. 
Additionally, an online replanning strategy is introduced to ensure that agents can adapt to dynamic environments in real time.

Main contributions of this work are three-fold: 
(i) a homotopy-aware path planning algorithm with online replanning that cooperatively plans each agent's global path, considering the environment’s topology and spatio-temporal conflicts for improved efficiency and robustness; 
(ii) a complete passage detection mechanism that identifies geometric details and a passage time map capturing the temporal behavior of agents, together providing a spatio-temporal representation of the global path;
(iii) extensive ablation studies show that the proposed method effectively prevents local congestion and deadlock, improving the average success rate in dense random scenarios from 4\%-13\% to over 90\%.

\section{Problem Formulation}

Consider a group of $N$ agents, each having a radius of $r$, navigating in a shared 2D workspace containing static obstacles. 
Each agent is capable of determining its own control input and communicating with each other.

\subsection{Global Path Planning}

Let $\mathcal{O} \subset \mathbb{R}^2$ denote the set of obstacles' occupied space. 
Similar to \cite{chen2023multi}, obstacles are assumed to be convex. 
This is because a nonconvex obstacle can be decomposed into several convex obstacles or enclosed within a larger convex one.
For each agent $i \in \mathcal{N} \triangleq \{ 1, 2, \cdots, N \}$, a reference path $\sigma^i: [0,1] \rightarrow \mathcal{F}$ is first generated to guide the agent's movement based on the global map, where $\mathcal{F} = \mathbb{R}^2 \setminus \mathcal{\tilde{O}}$ is the free space and $\tilde{\mathcal{O}}$ is $\mathcal{O}$ inflated by $r$.
Sampling-based methods (e.g., RRT*, PRM), search-based methods (e.g., Dijkstra, A*) and other sophisticated techniques can be employed to obtain a collision-free path for a single agent.
However, without appropriate global coordination, selfish agents may result in local congestion since they only follow their own optimal paths.
In this paper, our goal is to design a distributed global cooperative path planning method that fully utilizes the free space in the obstacle environment.

\subsection{Local Trajectory Planning} \label{sec:local_traj_plan}

\subsubsection{Agent Dynamics}

The dynamic model of agent $i \in \mathcal{N}$ is given by
$x_k^i(t) = \textbf{A} x_{k-1}^i(t) + \textbf{B} u_{k-1}^i(t),$
where $k \in \mathcal{K} = \{ 1, \cdots, K\}$ is the planning step.
$x_k^i(t) = [p_k^i(t), v_k^i(t)]$ is the planned state at time $t+kh$ for agent $i$, where $p_k^i(t)$ and $v_k^i(t)$ are the planned position and velocity, respectively, $h>0$ is the sampling time.  
$u_{k}^i(t)$ represents the planned control input. 
$\textbf{A} = \begin{bmatrix}
        \textbf{I}_2 & h\textbf{I}_2 \\
        \textbf{0}_2 & \textbf{I}_2
\end{bmatrix}$ 
and $\textbf{B} = \begin{bmatrix}
        \frac{h^2}{2}\textbf{I}_2 \\
        h \textbf{I}_2
\end{bmatrix}$ are the state transition matrix and control input matrix, respectively.
Moreover, the dynamical constraints are given by 
$\| {\bf\Theta}_u  u_{k-1}^i(t) \|_2 \le u_{\text{max}}, k \in \mathcal{K}, \| {\bf\Theta}_v  v_{k}^i(t) \|_2 \le v_{\text{max}}, k \in \mathcal{K},$
where ${\bf\Theta}_u$ and ${\bf\Theta}_v$ are given positive-definite matrices, and $u_{\text{max}}$ and $v_{\text{max}}$ are the maximum control input and velocity, respectively.

\subsubsection{Inter-agent Collision Avoidance} \label{sec:collision-avoidance} 

Each agent $i \in \mathcal{N}$ can be represented as a circle $\mathcal{R}^i = \{p^i + x \ |\  \| x \|_2 \le r \}$. 
The collision avoidance constraints between different agents can be directly formulated as 
$\| p^i - p^j \| \ge 2r$, where $p^i, p^j$ are positions of agents $i$ and $j$, respectively.

\subsubsection{Obstacle Avoidance} \label{sec:obstacle-avoidance}
The constraints of obstacle avoidance require that each agent does not collide with any obstacle in the environment,
that is, 
$\mathcal{R}^i \cap \mathcal{O} = \emptyset, \forall i \in \mathcal{N}$.

\subsection{Problem Statement}
Assuming that agents are all collision-free at the initial time $t_0$, our objective is to generate safe trajectories for all agents to their respective targets $p^i_{\text{target}}, \forall i \in \mathcal{N}$, while respecting constraints mentioned above. 
In the following sections, we mainly focus on a distributed framework that promotes cooperation among agents in both global path planning and local trajectory planning.

\section{Proposed Method} 

\begin{figure*} [t]
	\centering
	\includegraphics[width=\linewidth]{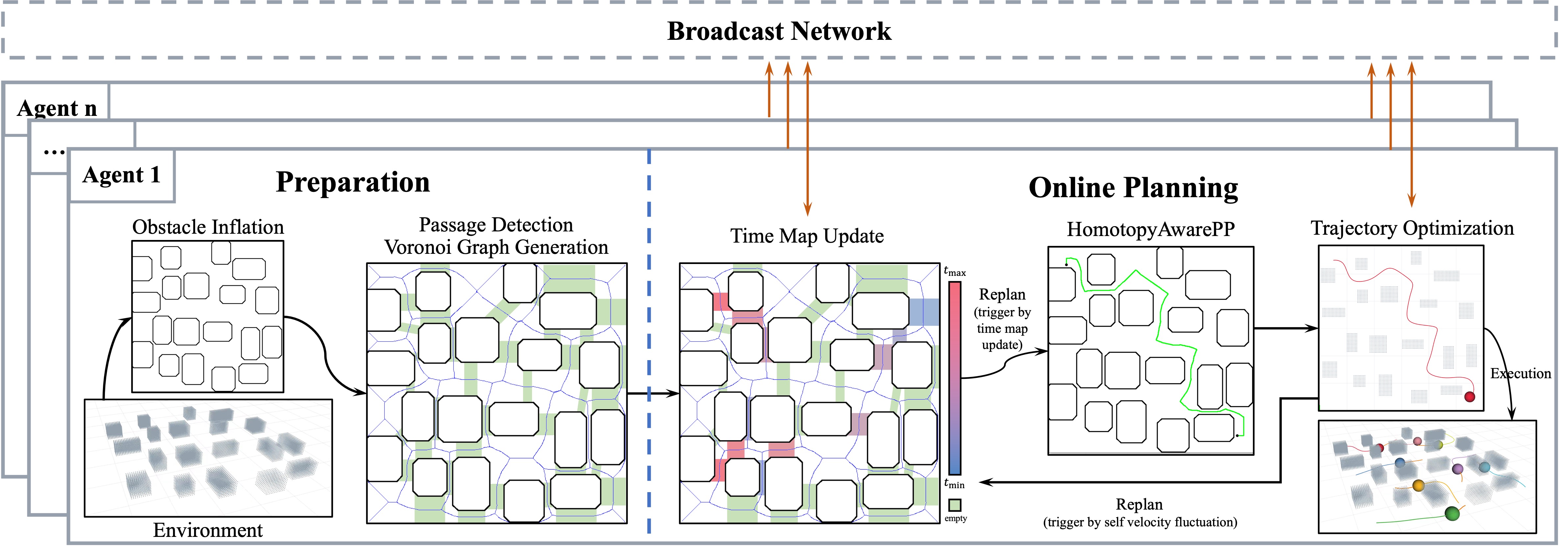}
	\caption{Illustration of the proposed framework, which consists of a map preparation module and an online planning module.}
	\label{fig:framework}
\end{figure*}

As shown in Fig.~\ref{fig:framework}, our framework consists of two modules: map preparation and online planning.
The map preparation module acquires the environment map, detects complete passages (Sec.~\ref{sec:complete-passages-detection}), and generates the Voronoi graph (Sec.~\ref{sec:path-search-on-voronoi-graph}), all in a distributed manner by each agent.
The online planning module generates the global path and local trajectory for each agent. Agents share and update their passage time maps (\ref{sec:passage-time-map}) via a broadcast network.
Based on this, the homotopy-aware optimal path planning algorithm (Sec.~\ref{sec:homotopy-aware-optimal-a-star}) is employed to generate the global path, taking into account the environment's topological structure and potential spatio-temporal conflicts with other agents.
Each agent then generates its trajectory through trajectory optimization (Sec.~\ref{sec:trajectory-optimization}) and broadcasts its trajectory to others. 
An online replanning strategy (Sec.~\ref{sec:ReplanCheck}) ensures real-time adaptation to environments.

\subsection{Homotopy-Aware Optimal Path Planning} \label{sec:homotopy-aware-optimal-path-planning}

\subsubsection{Path Search on Voronoi Graphs} \label{sec:path-search-on-voronoi-graph}


In obstacle-dense environments, multi-agent coordination requires selecting appropriate paths.
A Voronoi diagram, as proposed in \cite{rosmann2017integrated}, partitions the 2D workspace into cells based on Euclidean distance to obstacles within the inflated obstacle environment $\mathcal{\tilde{O}}$.
This environment $\mathcal{\tilde{O}}$ consists of several convex obstacles $\mathcal{\tilde{O}}^s$ ($s=1,2, \cdots, | \mathcal{\tilde{O}} |$), where $| \mathcal{\tilde{O}} |$ denotes the number of obstacles.
Each cell $\mathcal{C}^s$ contains points closer to a specific obstacle $\mathcal{\tilde{O}}^s$ than any other. The Voronoi graph $\mathcal{V}$ is generated by discretizing the boundaries of these cells, with each discretized point becoming a vertex $\zeta$, and edges connecting consecutive vertices. To find a path between a start point $p_{\text{start}}$ and target $p_{\text{target}}$, the nearest vertices $\zeta_{\text{start}}$ and $\zeta_{\text{target}}$ are identified. Collision-free edges are generated between $p_{\text{start}}$ (or $p_{\text{target}}$) and $\zeta_{\text{start}}$ (or $\zeta_{\text{target}}$), and search algorithms like A* or Dijkstra can be used to find the path on $\mathcal{V}$.

\subsubsection{Homotopy-Aware Optimal A* Algorithm} \label{sec:homotopy-aware-optimal-a-star}

The traditional path search on $\mathcal{V}$ is not sufficient since it only considers the agent's own path.
In order to realize global spatio-temporal coordination, the homotopy-aware optimal A* algorithm is proposed.
For a search-based method like A*, a cost function is crucial for determining the final path.
In this paper, the cost function has the following form:
\begin{equation} \label{eq:cost-function}
        \hat{f}(\tau) = \text{Len}(\tau) + \text{Dist}(\tau) - \lambda_P f_P(\tau) + \lambda_H f_H(\tau),
\end{equation}
where $\tau \in [0,1]$ is the argument of path $\sigma^i$.
The first two terms are traditional A* cost function elements where 
$\text{Len}(\tau)$ is the accumulated length from $\sigma^i(0)$ to $\sigma^i(\tau)$
and $\text{Dist}(\tau)$ is the Euclidean distance from $\sigma^i(\tau)$ to $p^i_{\text{target}}$.

Inspired by \cite{huang2024homotopic}, the third term is added to penalize the minimum passage width passed by $\sigma^i$ so far.
As shown in Fig.~\ref{fig:framework}, the complete passage detection is conducted to identify passages formed between pairs of obstacles.
The set of complete passages passed by $\sigma^i$ from $\sigma^i(0)$ to $\sigma^i(\tau)$ is denoted as $\mathcal{P}^i(\tau)$. 
The $q$-th element is denoted as $\mathcal{P}^i(\tau,q)$.
$f_P(\tau) = \min \| \mathcal{P}^i(\tau) \|$ is the shortest  passage line segment passed by $\sigma^i$.
$\lambda_P > 0$ is the weight of this term.
This term encourages the agent to pass through wider passages.

The fourth term aims to cooperate with other agents' path.
When agent $i$ is planning, only the passage time maps of higher-priority agents (e.g., those with smaller indices) are considered. These maps reflect the expected times at which agents pass through respective passages.
The passage time map according to $\mathcal{P}^i(\tau)$ is $\mathcal{T}^i(\tau)$.
A conflict index $\mathcal{CI}(i, j, \mathcal{P}^i(\tau,q))$ is defined to measure the potential spatio-temporal conflict between agent $i$ and agent $j$ at the complete passage $\mathcal{P}^i(\tau,q)= \mathcal{P}(\mathcal{\tilde{O}}^s, \mathcal{\tilde{O}}^c)$, where both $s$ and $c$ serve as obstacle indices.
If agent $j$ does not pass through $\mathcal{P}^i(\tau,q)$, $\mathcal{CI}(i, j, \mathcal{P}^i(\tau,q)) = 0$.
Otherwise, let us first denote the time spans of agent $i$ and $j$ passing through $\mathcal{P}^i(\tau,q)$ as $\mathbf{T}^{i,s,c} = [t^{i,s,c}_1 , t^{i,s,c}_2]$ and $\mathbf{T}^{j,s,c} = [t^{j,s,c}_1 , t^{j,s,c}_2]$.
If $\mathbf{T}^{i,s,c} \cap \mathbf{T}^{j,s,c} \neq \emptyset$, then $\mathcal{CI}(i, j, \mathcal{P}^i(\tau,q)) = 1$.
If $\mathbf{T}^{i,s,c} \cap \mathbf{T}^{j,s,c} = \emptyset$, then the conflict index is calculated as
$\mathcal{CI}(i, j, \mathcal{P}^i(\tau,q)) =  \exp \{ \alpha \cdot \delta(\mathbf{T}^{i,s,c}, \mathbf{T}^{j,s,c}) \}$,
where $  \delta(\mathbf{T}^{i,s,c}, \mathbf{T}^{j,s,c}) = |\max(t^{i,s,c}_1 ,t^{j,s,c}_1) - \min(t^{i,s,c}_2,t^{j,s,c}_2)|$ and $\alpha < 0$ is a constant showing the tolerance of temporal conflict.
With conflict index, $f_H(\tau)$ can be formulated as
\begin{equation} \label{eq:homotopic-term}
        f_H(\tau) = \sum_{q=1}^{|\mathcal{P}^i(\tau)|} \sum_{j<i} \mathcal{CI}(i, j, \mathcal{P}^i(\tau,q)),
\end{equation}
where $|\mathcal{P}^i(\tau)|$ is the number of complete passages in $\mathcal{P}^i(\tau)$.

Although the cost function is redesigned, the proposed algorithm can still be applied to find the optimal path on $\mathcal{V}$.
\begin{lemma}
        The algorithm is guaranteed to find the optimal path on the Voronoi graph $\mathcal{V}$ under the cost function \eqref{eq:cost-function}.
\end{lemma}
\begin{proof}
        We can reformulate the cost function as 
        $\hat{f}(\tau) = \hat{g}(\tau) + \hat{h}(\tau)$,
        where $\hat{g}(\tau) = \text{Len}(\tau) - \lambda_P f_P(\tau) + \lambda_H f_H(\tau)$ and $\hat{h}(\tau) = \text{Dist}(\tau)$.
        $\hat{g}(\tau)$ is the cost of the path from $p_{\text{start}}$ to a specific vertex $\sigma^i(\tau)$ with minimum cost so far and $\hat{h}(\tau)$ is the estimate from that vertex to $p_{\text{target}}$.
        Let the actual cost of an optimal path from $\sigma^i(\tau)$ to $p_{\text{target}}$ be $h(\tau) =  h_1(\tau) + h_2(\tau) + h_3(\tau),$
        where $ h_1(\tau)=\text{Len}(1) - \text{Len}(\tau)$, $h_2(\tau)= - \lambda_P (f_P(1) - f_P(\tau))$ and $h_3(\tau) =  \lambda_H (f_H(1) - f_H(\tau))$.
        Since $h_1(\tau) \ge \hat{h}(\tau)$, 
        $h_2(\tau)$ and $h_3(\tau)$ are both non-negative, then $h(\tau) \ge \hat{h}(\tau)$. 
        According to Theorem 1 in~\cite{hart1968formal}, then A* is guaranteed to find the optimal path. 
\end{proof}

Another perspective to understand the proposed algorithm is to consider the concept of homotopy.
Homotopic paths are defined as follows \cite{bhattacharya2012topological}:
\begin{definition}
      Two paths $\sigma_1$ and $\sigma_2$ connecting the same start and target points are homotopic if and only if one can be continuously deformed into the other without intersecting any obstacle. The set of all homotopic paths is called the homotopy class. 
\end{definition}

\begin{figure} [t]
	\centering
	\includegraphics[width=\linewidth]{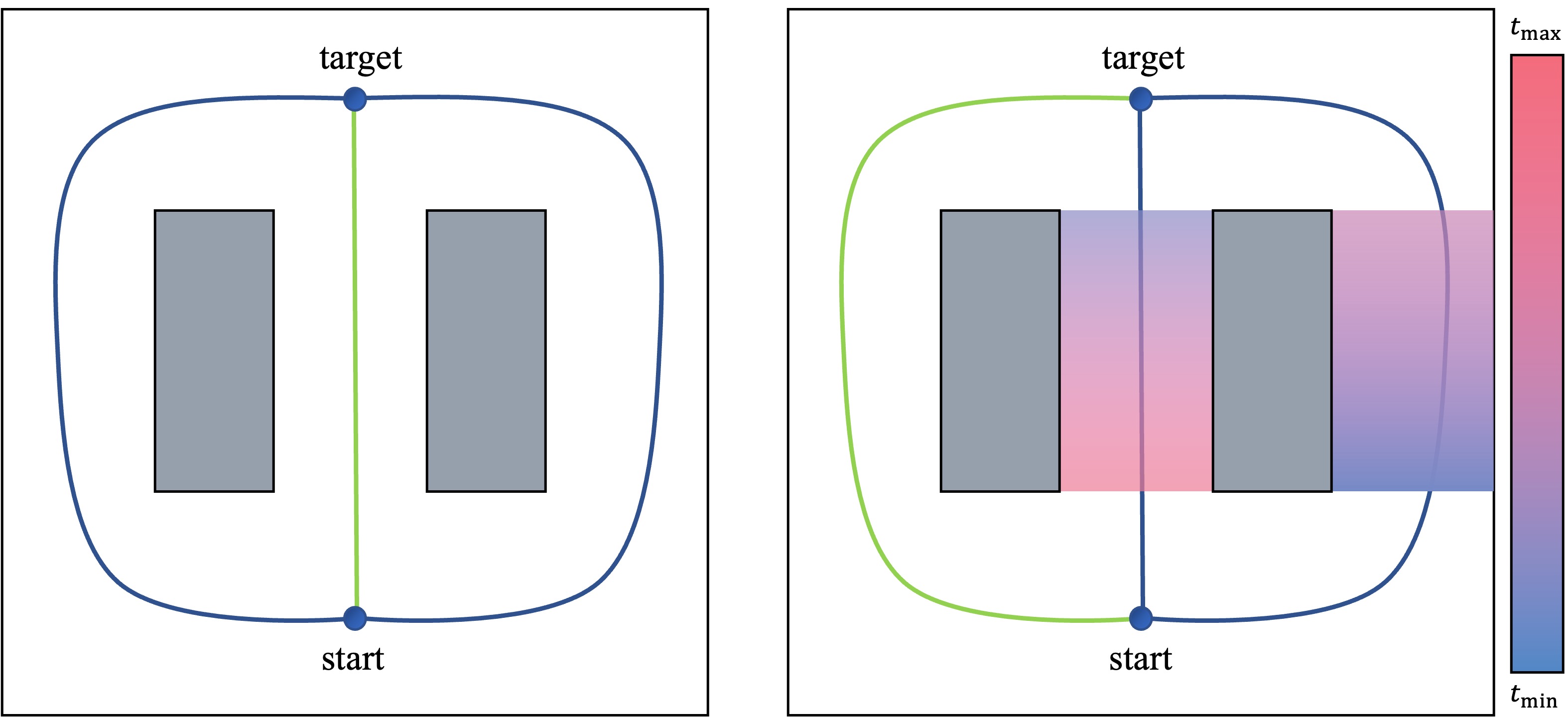}
	\caption{
    Illustration of time-aware homotopy classes for global coordination (from a single agent’s perspective).
    \textbf{Left}: Without considering temporal information, global coordination is difficult to achieve through homotopy classes.
    \textbf{Right}: Incorporating other agents’ passage time maps enables clear distinction of paths based on coordination. The gradient color represents the temporal behavior of other agents passing through passages.
    }
	\label{fig:time-homotopy}
\end{figure}

Global path coordination can be viewed as each agent selecting an appropriate path among all its homotopy classes.
The Voronoi graph $\mathcal{V}$ already provides simplified topological information about the environment.
Different paths belonging to distinct homotopy classes can be directly found on $\mathcal{V}$.
The proposed algorithm combines two steps together: searching paths in different homotopy classes on $\mathcal{V}$ and selecting the best path based on the cost function~\eqref{eq:cost-function}.
Since the optimality of the algorithm is guaranteed, there is no need to exhaustively search all homotopy classes.
Thus, the proposed algorithm can be considered \emph{homotopy-aware}.
Traditionally, homotopy is a purely spatial concept that ignores agents' temporal behavior. 
In this paper, we extend homotopy by incorporating temporal properties through a passage time map. 
Fig.~\ref{fig:time-homotopy} illustrates how this time-aware homotopy improves global path coordination, enabling spatio-temporal cooperation among agents.

\subsection{Complete Passages Detection and Passage Time Map} \label{sec:complete-passages-detection-and-passage-time-map}

In multi-agent systems, narrow spaces can cause spatio-temporal congestion. Thus, cooperatively utilizing free space is key to avoiding clustering. This section details two components from Sec.~\ref{sec:homotopy-aware-optimal-path-planning}—complete passage detection and the passage time map—to characterize the spatio-temporal properties of the global path.

\subsubsection{Complete Passages Detection} \label{sec:complete-passages-detection}

Each ordered pair of inflated obstacles $\mathcal{\tilde{O}}^s, \mathcal{\tilde{O}}^c \ (s<c)$ forms a generic passages, denoted as $(\mathcal{\tilde{O}}^s, \mathcal{\tilde{O}}^c)$.
In \cite{huang2024homotopic}, an extended visibility check in passage detection is proposed to obtain valid passages for free space determination.
It uses the shortest line segment between obstacles as a compact representation of the passage.
However, as illustrated in Fig.~\ref{fig:complete-passage-illustration}~(a-b), this results in the loss of geometric information about the passage, which cannot fully reflect the spatial behavior of agent passing through the passage.

\begin{figure} [t]
	\centering
	\includegraphics[width=\linewidth]{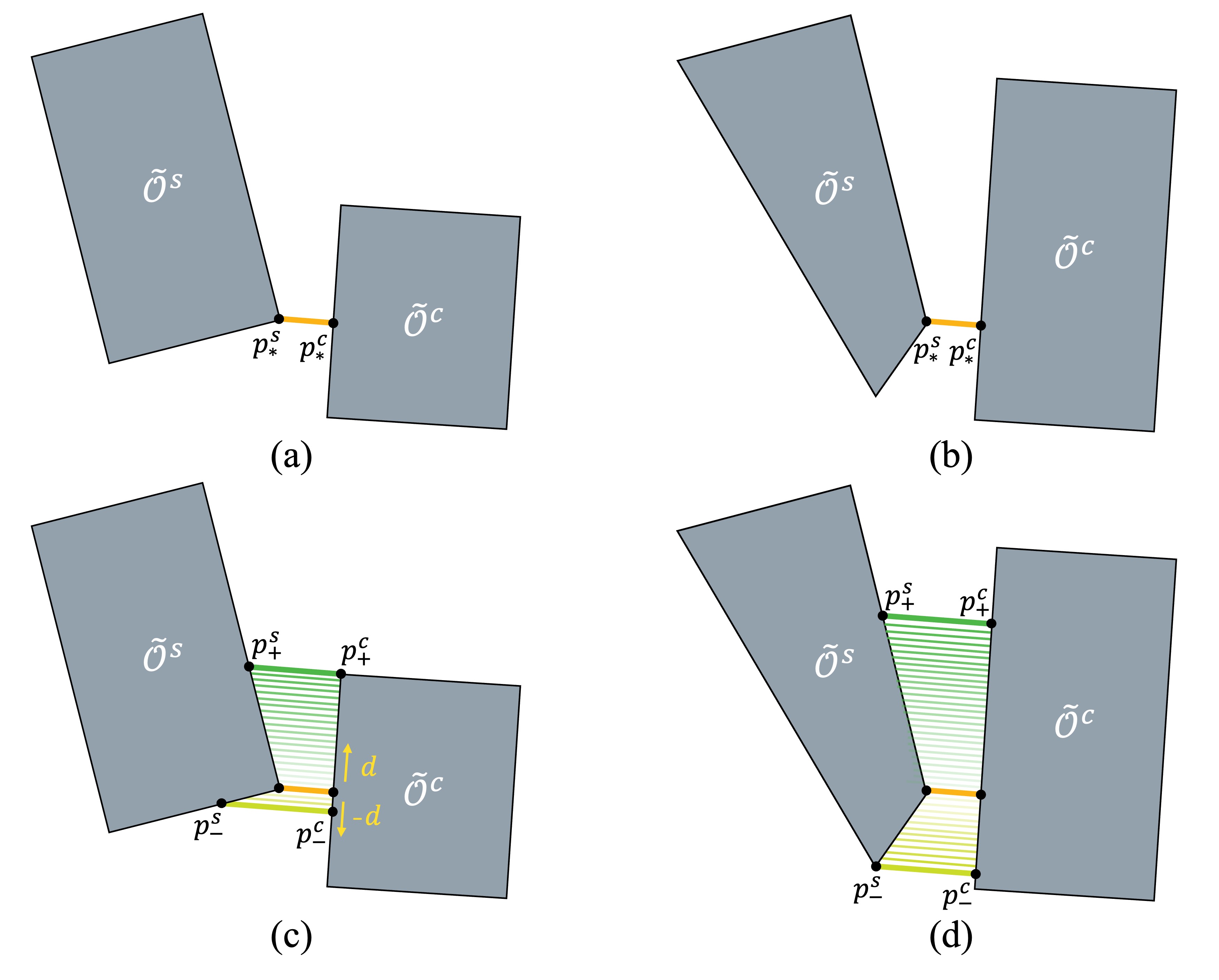}
	\caption{Illustration of complete passage detection. (a-b) Different pairs of obstacles have identical shortest line segments. (c-d) With complete passage detection, the spatial properties of the passage can be fully characterized.}
	\label{fig:complete-passage-illustration}
\end{figure}

To obtain complete passages, the segment is extended to capture additional geometric details.
Assume that the passage $(\mathcal{\tilde{O}}^s, \mathcal{\tilde{O}}^c)$ is valid after the extended visibility check.
Let us denote the shortest segment between $\mathcal{\tilde{O}}^s$ and $\mathcal{\tilde{O}}^c$ as $(p^s_*,p^c_*)$,
where 
$(p^s_*,p^c_*) = \argmin_{p^s \in \mathcal{\tilde{O}}^s, p^c \in \mathcal{\tilde{O}}^c} \| p^s - p^c \|_2,$
and $l(p^s_*,p^c_*)$ is the line passing through $p^s_*$ and $p^c_*$.
Based on the segment, a shifting direction vector $d$ satisfying $d \perp (p^s_*-p^c_*)$ is chosen.
The complete passage is constructed by shifting $l(p^s_*,p^c_*)$ continuously in the direction of both $d$ and $-d$ until one of three conditions is broken:
1) $l(p^s_* \pm \kappa d ,p^c_* \pm \kappa d) \cap \mathcal{\tilde{O}}^s \neq \emptyset$; 2) $l(p^s_* \pm \kappa d ,p^c_* \pm \kappa d) \cap \mathcal{\tilde{O}}^c \neq \emptyset$; 3) $\{p^s_\text{int}, p^c_\text{int}\} = l(p^s_* \pm \kappa d ,p^c_* \pm \kappa d) \cap (\mathcal{\tilde{O}}^s \cup \mathcal{\tilde{O}}^c), \| p^s_\text{int}- p^c_\text{int} \| \le l_\text{max}$,
where $\kappa > 0$ is the shifting step size and $l_\text{max}>0$ is the predefined maximum length of a passage segment.
The complete passage is denoted as 
\begin{equation} \label{eq:complete-passage}
        \mathcal{P}(\mathcal{\tilde{O}}^s, \mathcal{\tilde{O}}^c) = ((p^s_+,p^c_+), (p^s_*,p^c_*), (p^s_-,p^c_-)),
\end{equation}
in which $(p^s_+,p^c_+)$ and $(p^s_-,p^c_-)$ are the farthest segments that satisfy the above conditions (Fig.~\ref{fig:complete-passage-illustration}~(c-d)).
The representation of the complete passage can be interpreted as the \emph{entrance}, the \emph{narrowest} part and the \emph{exit} of the passage, respectively.

\subsubsection{Passage Time Map} \label{sec:passage-time-map}

After complete passage detection, the passage time map for agent $i$ can be generated to characterize the temporal property of its path $\sigma^i$.
Assume that the average velocity of agent $i$ is $\bar{v}^i$. 
Then, the passage time map $\mathcal{T}^i(1)$ can be constructed by dividing the accumulated distance from the start $\sigma^i(0)$ to the entrance or exit of each complete passage by $\bar{v}^i$.

\subsection{Trajectory Optimization} \label{sec:trajectory-optimization}

For the local trajectory planning, the agent needs to generate a safety-guaranteed trajectory to reach its target.
Firstly, the constraints introduced in Sec.~\ref{sec:collision-avoidance} and Sec.~\ref{sec:obstacle-avoidance} are unrealistic to be directly used.
The inter-agent avoidance constraints in Sec.~\ref{sec:collision-avoidance} can be equivalently converted to a linear constraint via the modified buffered Voronoi cells with warning band (MBVC-WB) as proposed in our previous work \cite{chen2024deadlock}:
${a_k^{ij}}^\top p_k^i \ge b_k^{ij}, \  \forall j \neq i, \  \forall k \in \mathcal{K},$
where $a_k^{ij}$ and $b_k^{ij}$ represent the normal vector and offset of the hyperplane, respectively.
The obstacle avoidance constraints from Sec.~\ref{sec:obstacle-avoidance} are reformulated using the safe corridor construction described in \cite{chen2023multi}. The constraints are expressed as a set of linear inequalities:
${ a_k^{i,s}}^\top p_k^i \ge b_k^{i,s}, \  k \in \mathcal{K},$
where $a_k^{i,s}$ and $b_k^{i,s}$ are the normal vector and offset of the hyperplane.

The time stamps of the reference path $\sigma^i$ are determined by the predefined average velocity $\bar{v}^i$.
To ensure accurate path following, a tractive point list is generated at every planning step.
Each round of trajectory optimization will calculate the planned states at $t+h, \cdots, t+Kh$.
Thus, the tractive point list consists of $K$ points in $\sigma^i$ with the same time stamps, denoted as $\tilde{p}^i_1, \cdots, \tilde{p}^i_K$.
The objective function is defined as 
\begin{equation} \label{eq:objective}
        C^i = \frac{1}{2} \sum_{k=1}^K \left( Q_k \| p_{k}^i - \tilde{p}^i_k \|^2 + R_{k-1} \| u_{k-1}^i \|^2 \right),
\end{equation}

The overall trajectory optimization is formulated as an optimal control problem to minimize \eqref{eq:objective} under the constraints in Sec.~\ref{sec:local_traj_plan} and \ref{sec:trajectory-optimization}.
By solving the optimal control problem, the agent can obtain the optimal trajectory $\mathcal{T}raj^i$.
This trajectory is then sent to the local controller to follow until a new trajectory is generated.

\subsection{Online Replanning} \label{sec:ReplanCheck}

If the predefined average velocity $\bar{v}^i$ is feasible, the agent can follow it in an obstacle-free environment using trajectory optimization. However, in obstacle-dense environments, the agent may need to slow down or stop to avoid collisions. The original passage time map $\mathcal{T}^i$ may mislead lower-priority agents, so global replanning (see Alg.~\ref{alg:replan-check}) is required to adapt to real-time changes. At the local level, agents plan their trajectories independently, with online replanning needed to ensure safety as the environment evolves.

\begin{algorithm}[t] \label{alg:replan-check}
	\caption{\texttt{ReplanCheck}}
	\SetKwInOut{Input}{Input}
	\Input{$\Delta t$, $\bar{v}^i_{\text{new}}$}

        \If{$\Delta t \ge \beta$}
        {
                $\mathcal{T}^i \leftarrow$ Update passage time map based on $\bar{v}^i_{\text{new}}$ \label{alg:replan-check:update-passage-time-map}

                Broadcast $\mathcal{T}^i$ to $\mathcal{N}^i$ \label{alg:replan-check:broadcast}

                $score \leftarrow$ Reevaluate $\sigma^i$ by calculating $f_H(1)$ according to \eqref{eq:homotopic-term} \label{alg:replan-check:reevaluate-path}

                \If{$score > \gamma$}
                {
                        \textbf{return} \textbf{True}
                }
                \Else{
                        Update time stamps of $\sigma^i$ based on $\bar{v}^i_{\text{new}}$ \label{alg:replan-check:update-reference-path}
                }
        }
        \Else{
             \If{\rm{Agent $j$ updates $\mathcal{T}^j$, $\exists j < i$}\label{alg:replan-check:high-priority-change}} 
             {
                        $score \leftarrow$ Reevaluate $\sigma^i$ by calculating $f_H(1)$ according to \eqref{eq:homotopic-term} \label{alg:replan-check:reevaluate-path-2}

                        \If{$score > \gamma$}
                        {
                                \textbf{return} \textbf{True} \label{alg:replan-check:reevaluate-path-2-return}
                        } 
             }
        }
        \textbf{return} \textbf{False}
\end{algorithm} 

In order to trigger the global path replanning, an indicator $\Delta t$ is first defined to measure the time error.
It is calculated by dividing the tracking error by the planned velocity at the next sampling time.
If $\Delta t$ exceeds a predefined threshold $\beta$, the agent will update its passage time map $\mathcal{T}^i$ based on the new average velocity $\bar{v}^i_{\text{new}}$ (line~\ref{alg:replan-check:update-passage-time-map}).
Then, the agent will broadcast $\mathcal{T}^i$ to its neighbors $\mathcal{N}^i$ (line~\ref{alg:replan-check:broadcast}).
Afterward, the agent will reevaluate its path by calculating $f_H(1)$ according to \eqref{eq:homotopic-term} (line~\ref{alg:replan-check:reevaluate-path}).
If the result score is higher than a predefined threshold $\gamma$, the agent will return \textbf{True} to trigger the replanning.
Otherwise, the agent will only update the time stamps of its own reference path $\sigma^i$ based on $\bar{v}^i_{\text{new}}$ (line~\ref{alg:replan-check:update-reference-path}).
If $\Delta t$ is less than $\beta$, the agent will check whether any agent with higher priority (i.e., agents with smaller indices) has updated its passage time map $\mathcal{T}^j$.
If so, the agent will reevaluate its path and decided whether return \textbf{True} to trigger the replanning (line~\ref{alg:replan-check:high-priority-change}-\ref{alg:replan-check:reevaluate-path-2-return}).
Other cases will return \textbf{False}.

Local trajectory replanning is performed using MPC at each replanning interval.
Through communication, agents can access the planned trajectories of other agents and update their own safety constraints.
A new trajectory optimization problem is then solved to generate a safe, updated trajectory.

\subsection{Complete Algorithm} \label{sec:complete-algorithm}

\begin{algorithm}[t] \label{alg:complete-algorithm}
	\caption{The Complete Algorithm for Agent $i$}
	\SetKwInOut{Input}{Input}
	\Input{$p^i(t_0)$, $p^i_{\text{target}}$, $\mathcal{O}$}

        $\mathcal{P} \leftarrow$ \texttt{CompletePassagesDetection}($\mathcal{O}$) \label{alg:complete-algorithm:complete-passages-detection}

        $\mathcal{V} \leftarrow$ \texttt{GenerateVoronoiGraph}($\mathcal{O}$) \label{alg:complete-algorithm:voronoi-graph}

        \textit{ReplanFlag} $\leftarrow$ \textbf{True} \label{alg:complete-algorithm:replan-flag-init}

        $\mathcal{T}, \ \mathcal{T}raj \leftarrow \emptyset$ \label{alg:complete-algorithm:traj-init}

	\While{\rm{agent $i$ not reaching its target} } 
	{
		\If{\textit{ReplanFlag} is \textnormal{\textbf{True}}}
                {
                        $\sigma^i, \ \mathcal{T}^i \leftarrow $ \texttt{HomotopyAwarePP}($p^i$, $p^i_{\text{target}}$,  $\mathcal{V}$, $\mathcal{P}$, $\mathcal{T}$) \label{alg:complete-algorithm:HomotopyAwarePP}

                        \textit{ReplanFlag} $\leftarrow$ \textbf{False}
                }

                $\mathcal{T}, \ \mathcal{T}raj \leftarrow $ \texttt{Communication}($\mathcal{N}^i$, $\mathcal{T}^i$, $\mathcal{T}raj^i$) \label{alg:complete-algorithm:communication}

                $\mathcal{T}raj^i \leftarrow $ \texttt{TrajOpt}($\mathcal{T}raj$, $\mathcal{O}$, $\sigma^i$) \label{alg:complete-algorithm:trajectory-optimization}

                \texttt{ExeTraj}($\mathcal{T}raj^i$) \label{alg:complete-algorithm:execute-trajectory}

                \textit{ReplanFlag} $\leftarrow$ \texttt{ReplanCheck}()  \label{alg:complete-algorithm:replan-check}
	}  
\end{algorithm}

The complete algorithm is proposed in Alg.~\ref{alg:complete-algorithm}. 
Initially, each agent preprocesses the environment to detect the complete set of passages, $\mathcal{P}$ (line~\ref{alg:complete-algorithm:complete-passages-detection}) and generate the Voronoi graph $\mathcal{V}$ (line~\ref{alg:complete-algorithm:voronoi-graph}). 
After planning begins, a boolean flag \textit{ReplanFlag} is checked to determine whether to apply the homotopy-aware optimal path planning (line~\ref{alg:complete-algorithm:HomotopyAwarePP}).
If the flag is set to \textbf{True}, the agent will plan its path.
Then, the agent communicates within its neighbor set $\mathcal{N}^i$ to get others' passage time maps $\mathcal{T}$ and trajectories $\mathcal{T}raj$ (line~\ref{alg:complete-algorithm:communication}).
Next, the agent will optimize its trajectory (line~\ref{alg:complete-algorithm:trajectory-optimization}) and executes its planned trajectory $\mathcal{T}raj^i$ (line~\ref{alg:complete-algorithm:execute-trajectory}).
Finally, it checks whether to replan (line~\ref{alg:complete-algorithm:replan-check}). 
The whole loop will be carried out indefinitely until this agent reaches its target position.

The computational complexity of the proposed method is as follows:
The offline phase includes complete passage detection and Voronoi graph generation.
Collision detection between two polygons has a complexity of $O(n_1 + n_2)$, where $n_1$ and $n_2$ are the number of vertices in the polygons \cite{gilbert2002fast}. 
Based on this, complete passage detection requires $\binom{|\mathcal{\tilde{O}}|}{2} \cdot |\mathcal{\tilde{O}}| + \frac{4n_p l_p}{\kappa}$ collision checks, where $|\mathcal{\tilde{O}}|$ is the number of obstacles, $n_p$ is the number of passages, and $l_p$ is the average passage length. It also requires $\binom{|\mathcal{\tilde{O}}|}{2} \cdot |\mathcal{\tilde{O}}|$ extended visibility checks, each with complexity $O(n_1)$.
Voronoi graph generation has a complexity of $O(m^2 \log m)$, where $m$ is the total number of edges of all obstacles.
The online planning process includes global path planning with complexity $O(n_{p}b^d)$, where $b$ is the branching factor and $d$ is the solution path depth, and final optimization, which involves a QCQP problem with $6K$ variables and $(3+N+n_{\text{near}}) \cdot K$ constraints, where $n_{\text{near}}$ is the number of obstacles around the agent.
Thanks to efficient optimization tools (e.g., \cite{Verschueren2021}), the optimization can be solved within tens of milliseconds.

\section{Numerical Experiments}

This section validates the proposed method through simulations and experiments.
We implemented the Voronoi graph generation based on the source code of \cite{binder2019multi}.
The communication between agents is implemented by  \texttt{ROS1}.
The optimal control problem is solved by \texttt{acados} \cite{Verschueren2021}.
Lastly, the local planner's ability to handle deadlock situations is enhanced by incorporating the deadlock resolution scheme from \cite{chen2024deadlock}.
All simulations are conducted on a desktop computer with an Intel Core i7-13700K @3.4GHz CPU and a 32GB RAM, 
employing multiple \texttt{ROS} nodes to simulate distinct agents.

\subsection{Implementation Details}
The proposed algorithm is built upon several key hyperparameters. 
The proper selection of these hyperparameters directly dictates user-defined system behaviors.
Moreover, these hyperparameters feature intuitive physical interpretations, which enables users to tune them in a straightforward manner.
$\bar{v}^i$ denotes a predefined average velocity. Setting it closer to the actual average velocity can enable agents to cooperate efficiently during the initial phase.
$l_\text{max}$ denotes the maximum passage width. To characterize narrow passages where agents cannot pass each other, $l_\text{max}$ can be set as twice the agent diameter.
A larger value of $\alpha < 0$ implies lower tolerance for temporal conflicts, while a smaller value implies higher tolerance.
$\beta$ defines the time threshold for triggering the update of the self-passage time map, while $\gamma$ serves as the threshold for initiating replanning.

During the following simulations, we set those hyperparameters as follows:
$\bar{v}^i= 0.5\text{m/s}^2$; $l_\text{max} = 0.8\text{m}$; $\alpha = -0.3$; $\beta = 1.0 \text{s}$; $\gamma = 1.5$.
We set the weight as $\lambda_P = 50.0$; $\lambda_H = 500.0$.
The trajectory optimization parameters are set as $K = 12$; $h = 0.2\text{s}$; $u_\text{max} = 2.0\text{m/s}^2$; $v_\text{max} = 1.0\text{m/s}^2$; $r = 0.2\text{m}$.
All parameters remain unchanged unless otherwise specified.

\subsection{Results}

\begin{figure}[t]
        \centering
        \includegraphics[width=\linewidth]{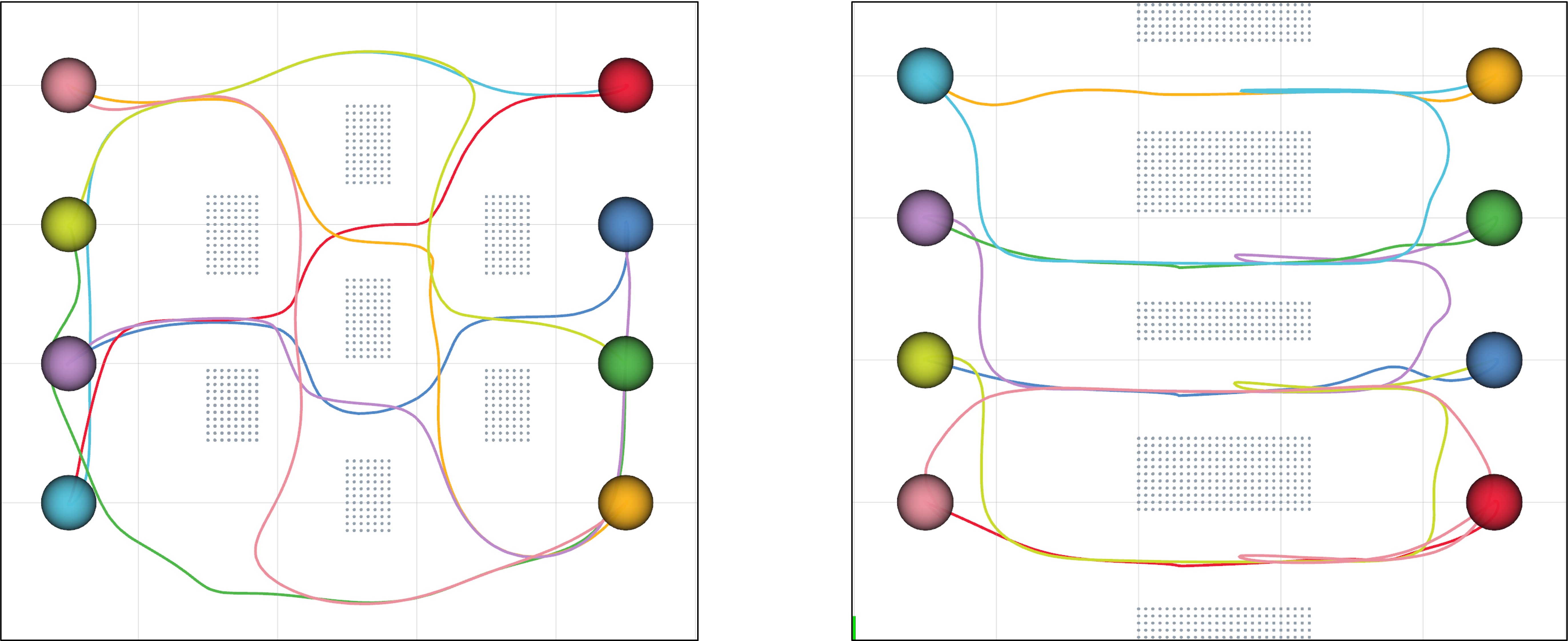}
        \caption{\textbf{Left}: 8 agents crossing through scattered obstacles. \textbf{Right}: 8 agents exchanging their positions through four corridors that only allow one agent to pass at a time.}
        \label{fig:results-demo}
\end{figure}



The first simulation involves 8 agents navigating through 7 scattered obstacles in a $5.0 \text{m} \times 4.5 \text{m}$ environment.
The agents are tasked with swapping positions with those at the opposite corners.
As shown in the left figure in Fig.~\ref{fig:results-demo}, 8 agents successfully reach their targets with full spatio-temporal coordination.
Each agent plans its global path in $3.1\text{ms}$ on average.
The total trajectory length is $52.1\text{m}$ and maximum time is $19.3\text{s}$. 
The time consumption of trajectory optimization at each replanning interval is $9.1\text{ms}$ on average.


To further evaluate scalability, a more complex $10\text{m}\times 10\text{m}$ environment with 35 obstacles is considered, where 16 agents exchange positions as shown in Fig.~\ref{fig:compare-trajectory}. Despite the larger map and more agents, the proposed method performs well. In offline stage, complete passage detection takes $1.5 \text{ms}$, and Voronoi graph generation takes $0.23 \text{s}$. The Voronoi graph reduces the search space, and homotopy-aware path planning ensures effective agent cooperation. On average, each agent plans its global path in $41.5\text{ms}$, with local trajectory optimization taking $19.8 \text{ms}$ per replanning interval.


A more challenging scenario is shown in the right figure of Fig.~\ref{fig:results-demo}.
In this case, 8 agents are required to exchange positions through four narrow corridors that only allow one agent to pass at a time. 
We set $\gamma = 0.9$ to lower the threshold for replanning.
Agents must replan their paths, as a one-shot planning approach could lead to severe congestion.
With the aid of complete passage detection and the passage time map, agents can successfully detect potential spatio-temporal conflicts and replan their paths in real-time.

\subsection{Comparison with Baselines}

\begin{figure*}[t]
        \centering
        \includegraphics[width=\textwidth]{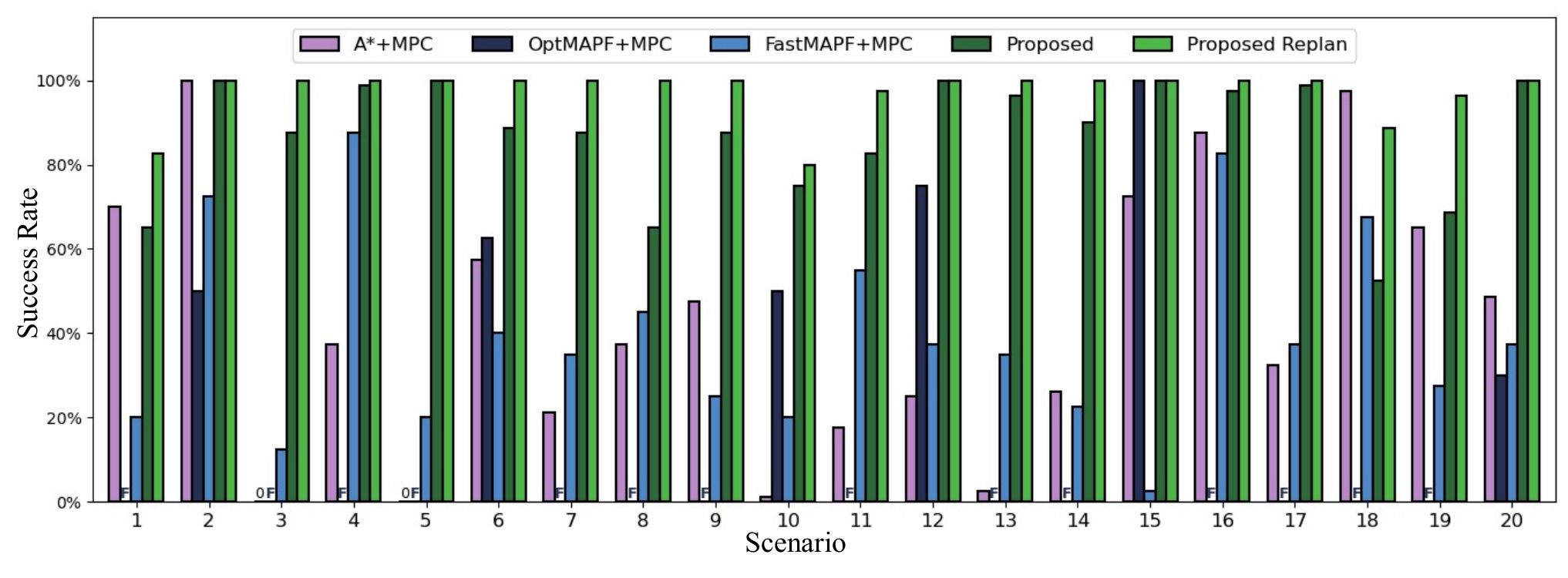}
        \caption{Swarm success rate results. The proposed method maintains a high success rate across all scenarios. With online replanning, its robustness is further enhanced.
        ``F” indicates that the global planner fails to obtain a result within 60 minutes. ``0'' denotes zero percent success rate. All data represents the average of 10 independent runs.}
        \label{fig:random-sim-results}
\end{figure*}

Then, the overall performance of the proposed method is evaluated.
In the following simulations, we keep the local planner unchanged (a MPC-based method \cite{chen2023multi}) and replace the global planner with the single-agent A*, an optimal MAPF, a fast MAPF and the proposed homotopy-aware optimal path planning (with and without online global replanning).
Conflict-based search (CBS)~\cite{sharon2015conflict} is an optimal MAPF algorithm that finds the optimal path for each agent in a shared environment. However, its high computational complexity makes it unsuitable for real-time applications. In contrast, fast MAPF algorithms provide quicker results at the cost of optimality. Prioritized planning~\cite{silver2005cooperative} is used here, which plans paths based on agent priorities, though it is neither optimal nor complete~\cite{stern2019multi}.

\begin{table}[t]
\centering
\caption{Comparison of different methods.}
\label{tab:metrics-comparison}
\begin{threeparttable}
\begin{tabular}{@{}cccccc@{}}
\toprule
\multirow{2}{*}{} & \multicolumn{3}{c}{Scenario Success Rate [\%]} & Maximum & Total \\ \cmidrule(lr){2-4}
                        & avg.    & max.     & min.    &        Time [s]                &        Length [m]                 \\ \midrule
A*+MPC            & 13.0    & 20.0     & 10.0    & 36.4                     & 58.6                    \\
OptMAPF+MPC       & 5.0     & 5.0      & 5.0     & 25.5                     & 56.5                   \\
FastMAPF+MPC      & 4.0     & 5.0      & 0.0     & 35.1                     & 60.3                   \\
Proposed          & 48.5    & 55.0     & 40.0    & 30.0                     & 68.4                   \\
Proposed (replan) & 90.5    & 100.0    & 85.0    & 38.6                     & 70.1                \\ \bottomrule
\end{tabular}
\begin{tablenotes}
        \footnotesize
        \item[*] All data represents the average of 10 independent runs. Each run contains 20 scenarios. 
\end{tablenotes}
\end{threeparttable}
\end{table}

Twenty distinct $6\text{m}\times 6\text{m}$ crowded scenarios, similar to the map in Fig.~\ref{fig:framework}, are randomly generated with 19 obstacles.
In each scenario, eight agents are placed at fixed start position and required to swap their positions with the ones on the opposite side.
A scenario is considered failed if any agent fails to reach its target within 100s.
The following four metrics are recorded and depicted in Table~\ref{tab:metrics-comparison} and Fig.~\ref{fig:random-sim-results}:
\begin{itemize}
        \item Scenario success rate: the percentage of successful scenarios in total 20 scenarios.
        \item Swarm success rate: the percentage of successful agents in the swarm in each scenario.
        \item Average total trajectory length of successful scenarios.
        \item Average maximum time of successful scenarios.
\end{itemize}
The results show that the proposed method achieves the highest scenario and swarm success rates, with online replanning further enhancing its adaptability. 
Unlike single-agent A*, it considers other agents' passage time maps, improving spatio-temporal coordination and reducing congestion.
CBS fails in 14 out of 20 scenarios due to the high computational cost of finding optimal solutions, while prioritized planning finds solutions for all scenarios in reasonable time.
However, both optimal and fast MAPF methods still suffer from the gap between path and motion planning.
As illustrated in Fig.~\ref{fig:compare-trajectory}, agents easily get stuck in some narrow passages, due to the failure of local planner to exactly follow the global path time stamps. 
On the other hand, the proposed method addresses this by relaxing strict collision avoidance in the global planner and using soft constraints for better spatio-temporal coordination, while the local planner ensures safety and robustness.

\begin{figure} [t]
	\centering
	\includegraphics[width=\linewidth]{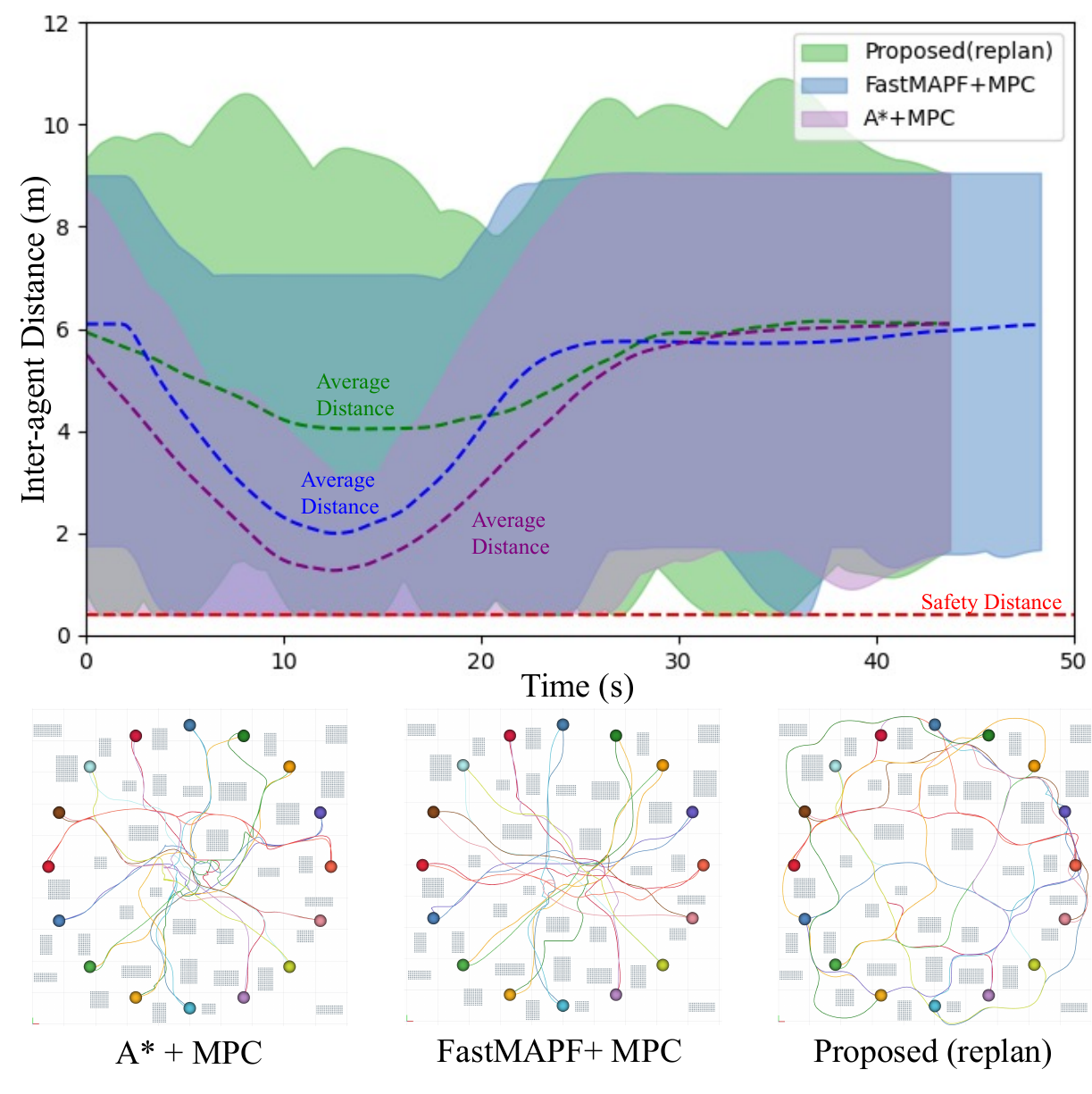}
	\caption{The inter-agent distances of three methods and their corresponding trajectories. 
    The proposed method fully leverages the topological structure, maintaining the highest average distances throughout most of the time.}
	\label{fig:inter-distance}
\end{figure}

As shown in Fig.~\ref{fig:inter-distance}, another scenario with a setup similar to that in Fig.~\ref{fig:compare-trajectory} is designed to demonstrate the system's ability to fully utilize the available free space for spatio-temporal coordination.
Although all three methods tested succeed, they yield different inter-agent distances.
Notably, the proposed method maintains the largest average inter-agent distance, offering a distinct advantage in preventing local congestion.
The core idea of our method is to trade additional space for improved efficiency. As shown in Table~\ref{tab:metrics-comparison}, longer total trajectory lengths and larger maximum times are observed in the proposed method, representing a trade-off for better spatio-temporal coordination.

\subsection{Evaluation on Online Global Replanning}

\begin{figure} [t]
	\centering
	\includegraphics[width=\linewidth]{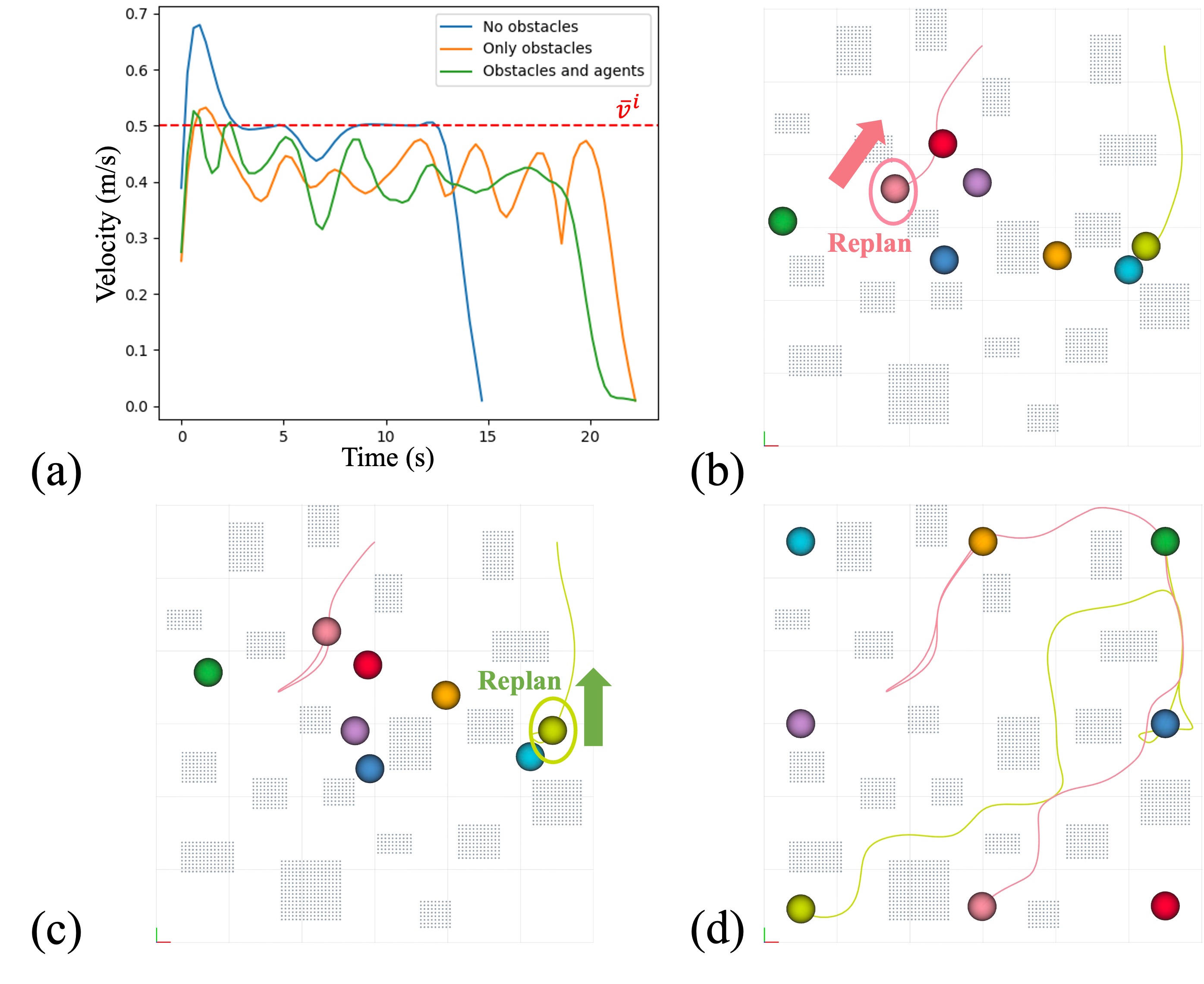}
	\caption{(a) An example of agent's velocity fluctuation in free space, obstacle space and obstacle space with other agents. (b-d) Snapshots showing that two agents trigger global replanning in a crowded scenario.}
	\label{fig:replanning-simulation}
\end{figure}

Finally, the online global replanning is tested by ablation study.
As mentioned in Sec.~\ref{sec:ReplanCheck}, a relatively accurate passage time map is crucial for reaching global coordination.
However, the agent's actual velocity may fluctuate in an obstacle-dense environment (see Fig.~\ref{fig:replanning-simulation}~(a)).
To evaluate the proposed mechanism, we equipped our methods with no replanning, pure replanning (start replanning when $\Delta t \ge \beta$) and the proposed online global replanning.
Three candidates are tested in a same $6\text{m}\times 6\text{m}$ crowded scenario with 19 randomly generated obstacles (as shown in Fig.~\ref{fig:replanning-simulation}~(b-d)).
Table~\ref{table:comparison-replanning-check} shows their planning results with 4, 6 and 8 agents.
From the table, we can see that when the number of agents increases, the proposed online global replanning can still help to maintain a high success rate while not triggering replanning frequently.

Fig.~\ref{fig:replanning-simulation}~(b-d) illustrate the process of the proposed replanning method.
In (b), the circled pink agent initiates replanning due to updates in the passage time maps of higher priority agents.
In (c), the circled light green agent starts replanning because of its own velocity slowdown and a potential blockage with another agent.
(d) shows the overall trajectories of the two agents after replanning.

\begin{table}[t]
\centering
\caption{Comparison of different replanning check methods.}
\vspace{-0.2cm}
\label{table:comparison-replanning-check}
\begin{threeparttable}
\begin{tabular}{@{}cccccc@{}}
        \toprule
        Num of Agents  & Methods & Success & $N_r$ & $L$ & $T$ \\ \midrule
        \multirow{3}{*}{4} & No replan          & 5/5  &  -   &   40.3     &   30.3    \\
        & Pure               &  5/5  &     2       &         40.5           &          30.9         \\
        & Proposed           &     5/5     &      0      &         40.5          &        30.3           \\ \midrule
        \multirow{3}{*}{6} & No replan          &    4/5      &     -       &          60.4          &        30.3           \\
        & Pure               &     5/5     &      28      &         64.0           &        53.6           \\
        & Proposed           &    5/5      &     15       &         61.7           &         39.4          \\ \midrule
        \multirow{3}{*}{8} & No replan          &     0/5     &      -      &        -            &         -          \\
        & Pure               &     0/5     &      -      &          -          &        -           \\
        & Proposed           &     5/5     &     29       &         77.7           &        41.7           \\ \bottomrule
\end{tabular}
\footnotesize
$N_r$: total number of replanning;  $L[{\rm m}]$: total trajectory length;
 $T[{\rm s}]$: maximum time;
 all data represents the average of 5 independent runs. 
\end{threeparttable}
\vspace{-0.2cm}
\end{table}

\subsection{Experiment Results}

\begin{figure} [t]
	\centering
	\includegraphics[width=\linewidth]{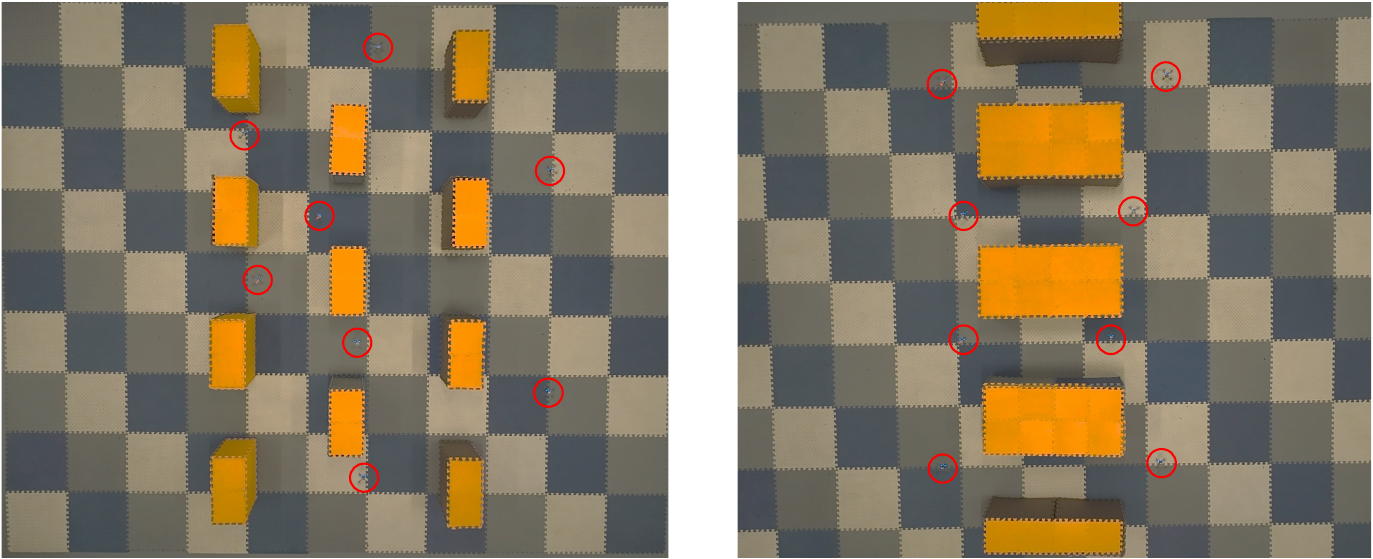}
	\caption{Real-world experiments. \textbf{Left}: Eight Crazyflies navigating and exchanging positions in a scattered environment. \textbf{Right}: Eight agents traversing four narrow passages.}
	\label{fig:exp-fig}
\end{figure}

For experiments, all agent computations are executed on a central computer using multiple \texttt{ROS} nodes to simulate distinct agents and transmit control commands to individual Crazyflie drones.
All hardware experiments are executed on the platform of crazyswarm \cite{preiss2017crazyswarm} with the help of a motion capture system OptiTrack.

As shown in Fig.~\ref{fig:exp-fig}, in the first experiment, eight Crazyflie drones perform a crossing maneuver in an area with 11 obstacles, each measuring $0.6\text{m} \times 0.3\text{m}$. Our proposed method effectively facilitates global coordination among the drones, significantly reducing local congestion and blockages.
In the second experiment, eight Crazyflie drones must swap positions with those on the opposite side by navigating through four corridors, with only one drone allowed to pass through each corridor at a time.
In this scenario, it is impossible to plan a globally coordinated spatio-temporal path in one go. However, through replanning, the drones can coordinate and adjust their paths online, ultimately completing the task.
Full demonstration can be found in the supplementary video.

\section{Conclusion}

This paper presents a novel two-layer distributed trajectory planning framework for multi-agent systems in obstacle-dense environments. 
The global layer uses homotopy-aware optimal path planning for efficient coordination and conflict reduction, while the local layer employs MPC-based trajectory optimization for safety. 
An online replanning strategy ensures real-time adaptability. 
Simulations and experiments demonstrate its effectiveness.
Future work will focus on extending the framework to more complex systems and more dynamic scenarios.

\addtolength{\textheight}{-12cm}   








\bibliographystyle{IEEEtran}
\bibliography{ref}

\begin{thebibliography}{10}
\providecommand{\url}[1]{#1}
\csname url@samestyle\endcsname
\providecommand{\newblock}{\relax}
\providecommand{\bibinfo}[2]{#2}
\providecommand{\BIBentrySTDinterwordspacing}{\spaceskip=0pt\relax}
\providecommand{\BIBentryALTinterwordstretchfactor}{4}
\providecommand{\BIBentryALTinterwordspacing}{\spaceskip=\fontdimen2\font plus
\BIBentryALTinterwordstretchfactor\fontdimen3\font minus
  \fontdimen4\font\relax}
\providecommand{\BIBforeignlanguage}[2]{{%
\expandafter\ifx\csname l@#1\endcsname\relax
\typeout{** WARNING: IEEEtran.bst: No hyphenation pattern has been}%
\typeout{** loaded for the language `#1'. Using the pattern for}%
\typeout{** the default language instead.}%
\else
\language=\csname l@#1\endcsname
\fi
#2}}
\providecommand{\BIBdecl}{\relax}
\BIBdecl

\bibitem{tordesillas2021mader}
J.~Tordesillas and J.~P. How, ``Mader: Trajectory planner in multiagent and
  dynamic environments,'' \emph{IEEE Transactions on Robotics}, vol.~38, no.~1,
  pp. 463--476, 2021.

\bibitem{park2022online}
J.~Park, D.~Kim, G.~C. Kim, D.~Oh, and H.~J. Kim, ``Online distributed
  trajectory planning for quadrotor swarm with feasibility guarantee using
  linear safe corridor,'' \emph{IEEE Robotics and Automation Letters}, vol.~7,
  no.~2, pp. 4869--4876, 2022.

\bibitem{grover2023before}
J.~Grover, C.~Liu, and K.~Sycara, ``The before, during, and after of
  multi-robot deadlock,'' \emph{The International Journal of Robotics
  Research}, vol.~42, no.~6, pp. 317--336, 2023.

\bibitem{chen2023multi}
Y.~Chen, C.~Wang, M.~Guo, and Z.~Li, ``Multi-robot trajectory planning with
  feasibility guarantee and deadlock resolution: An obstacle-dense
  environment,'' \emph{IEEE Robotics and Automation Letters}, vol.~8, no.~4,
  pp. 2197--2204, 2023.

\bibitem{zhou2021raptor}
B.~Zhou, J.~Pan, F.~Gao, and S.~Shen, ``Raptor: Robust and perception-aware
  trajectory replanning for quadrotor fast flight,'' \emph{IEEE Transactions on
  Robotics}, vol.~37, no.~6, pp. 1992--2009, 2021.

\bibitem{park2023decentralized}
J.~Park, I.~Jang, and H.~J. Kim, ``Decentralized deadlock-free trajectory
  planning for quadrotor swarm in obstacle-rich environments,'' in \emph{2023
  IEEE International Conference on Robotics and Automation (ICRA)}.\hskip 1em
  plus 0.5em minus 0.4em\relax IEEE, 2023, pp. 1428--1434.

\bibitem{cavorsi2023multirobot}
M.~Cavorsi, L.~Sabattini, and S.~Gil, ``Multirobot adversarial resilience using
  control barrier functions,'' \emph{IEEE Transactions on Robotics}, vol.~40,
  pp. 797--815, 2023.

\bibitem{stern2019multi}
R.~Stern, ``Multi-agent path finding--an overview,'' \emph{Artificial
  Intelligence: 5th RAAI Summer School, Dolgoprudny, Russia, July 4--7, 2019,
  Tutorial Lectures}, pp. 96--115, 2019.

\bibitem{hou2022enhanced}
J.~Hou, X.~Zhou, Z.~Gan, and F.~Gao, ``Enhanced decentralized autonomous aerial
  robot teams with group planning,'' \emph{IEEE Robotics and Automation
  Letters}, vol.~7, no.~4, pp. 9240--9247, 2022.

\bibitem{bhattacharya2012topological}
S.~Bhattacharya, M.~Likhachev, and V.~Kumar, ``Topological constraints in
  search-based robot path planning,'' \emph{Autonomous Robots}, vol.~33, pp.
  273--290, 2012.

\bibitem{de2024topology}
O.~De~Groot, L.~Ferranti, D.~M. Gavrila, and J.~Alonso-Mora, ``Topology-driven
  parallel trajectory optimization in dynamic environments,'' \emph{IEEE
  Transactions on Robotics}, 2024.

\bibitem{huang2024homotopic}
J.~Huang, Y.~Tang, and K.~W.~S. Au, ``Homotopic path set planning for robot
  manipulation and navigation,'' in \emph{Robotics: Science and Systems}, 2024.

\bibitem{mao2024optimal}
P.~Mao, R.~Fu, and Q.~Quan, ``Optimal virtual tube planning and control for
  swarm robotics,'' \emph{The International Journal of Robotics Research},
  vol.~43, no.~5, pp. 602--627, 2024.

\bibitem{zhou2021ego}
X.~Zhou, J.~Zhu, H.~Zhou, C.~Xu, and F.~Gao, ``Ego-swarm: A fully autonomous
  and decentralized quadrotor swarm system in cluttered environments,'' in
  \emph{2021 IEEE international conference on robotics and automation
  (ICRA)}.\hskip 1em plus 0.5em minus 0.4em\relax IEEE, 2021, pp. 4101--4107.

\bibitem{kasaura2023homotopy}
K.~Kasaura, ``Homotopy-aware multi-agent path planning on plane,'' \emph{arXiv
  preprint arXiv:2310.01945}, 2023.

\bibitem{rosmann2017integrated}
C.~R{\"o}smann, F.~Hoffmann, and T.~Bertram, ``Integrated online trajectory
  planning and optimization in distinctive topologies,'' \emph{Robotics and
  Autonomous Systems}, vol.~88, pp. 142--153, 2017.

\bibitem{hart1968formal}
P.~E. Hart, N.~J. Nilsson, and B.~Raphael, ``A formal basis for the heuristic
  determination of minimum cost paths,'' \emph{IEEE transactions on Systems
  Science and Cybernetics}, vol.~4, no.~2, pp. 100--107, 1968.

\bibitem{chen2024deadlock}
Y.~Chen, M.~Guo, and Z.~Li, ``Deadlock resolution and recursive feasibility in
  mpc-based multi-robot trajectory generation,'' \emph{IEEE Transactions on
  Automatic Control}, 2024.

\bibitem{gilbert2002fast}
E.~G. Gilbert, D.~W. Johnson, and S.~S. Keerthi, ``A fast procedure for
  computing the distance between complex objects in three-dimensional space,''
  \emph{IEEE Journal on Robotics and Automation}, vol.~4, no.~2, pp. 193--203,
  2002.

\bibitem{Verschueren2021}
R.~Verschueren, G.~Frison, D.~Kouzoupis, J.~Frey, N.~v. Duijkeren, A.~Zanelli,
  B.~Novoselnik, T.~Albin, R.~Quirynen, and M.~Diehl, ``acados—a modular
  open-source framework for fast embedded optimal control,'' \emph{Mathematical
  Programming Computation}, vol.~14, no.~1, pp. 147--183, 2022.

\bibitem{binder2019multi}
B.~Binder, F.~Beck, F.~K{\"o}nig, and M.~Bader, ``Multi robot route planning
  (mrrp): Extended spatial-temporal prioritized planning,'' in \emph{2019
  IEEE/RSJ International Conference on Intelligent Robots and Systems
  (IROS)}.\hskip 1em plus 0.5em minus 0.4em\relax IEEE, 2019, pp. 4133--4139.

\bibitem{sharon2015conflict}
G.~Sharon, R.~Stern, A.~Felner, and N.~R. Sturtevant, ``Conflict-based search
  for optimal multi-agent pathfinding,'' \emph{Artificial intelligence}, vol.
  219, pp. 40--66, 2015.

\bibitem{silver2005cooperative}
D.~Silver, ``Cooperative pathfinding,'' in \emph{Proceedings of the aaai
  conference on artificial intelligence and interactive digital entertainment},
  vol.~1, no.~1, 2005, pp. 117--122.

\bibitem{preiss2017crazyswarm}
J.~A. Preiss, W.~Honig, G.~S. Sukhatme, and N.~Ayanian, ``Crazyswarm: A large
  nano-quadcopter swarm,'' in \emph{2017 IEEE International Conference on
  Robotics and Automation (ICRA)}.\hskip 1em plus 0.5em minus 0.4em\relax IEEE,
  2017, pp. 3299--3304.

\end{thebibliography}

\end{document}